\documentclass{article}

\usepackage{microtype}
\usepackage{graphicx}
\usepackage{subcaption}
\usepackage{float} 
\usepackage{booktabs} 
\usepackage{multirow}
\usepackage{colortbl}

\usepackage{hyperref}

\usepackage[preprint]{icml2026}

\usepackage{amsmath}
\usepackage{amssymb}
\usepackage{mathtools}
\usepackage{amsthm}
\usepackage{bbm}
\usepackage{pifont}

\usepackage[capitalize,noabbrev]{cleveref}

\usepackage{tcolorbox}
\tcbuselibrary{skins}

\definecolor{facOrange}{HTML}{FAC074}
\definecolor{skyBlue}{HTML}{6C8EBF}

\newtcolorbox{defaultbox}{
  enhanced,
  arc=4mm,
  colback=gray!10,
  colframe=black,
  boxrule=1pt,
  left=6pt, right=6pt,
  top=6pt, bottom=6pt,
}

\newtcolorbox{orangebox}{
  enhanced,
  arc=4mm,
  colback=gray!10,
  colframe=facOrange,
  boxrule=1pt,
  left=6pt, right=6pt,
  top=6pt, bottom=6pt,
}

\newtcolorbox{bluebox}{
  enhanced,
  arc=4mm,
  colback=gray!10,
  colframe=skyBlue,
  boxrule=1pt,
  left=6pt, right=6pt,
  top=6pt, bottom=6pt,
}

\theoremstyle{plain}
\newtheorem{theorem}{Theorem}[section]
\newtheorem{proposition}[theorem]{Proposition}
\newtheorem{lemma}[theorem]{Lemma}

\theoremstyle{definition}
\newtheorem{definition}[theorem]{Definition}

\theoremstyle{remark}
\newtheorem{remark}[theorem]{Remark}

\usepackage[textsize=tiny]{todonotes}

\icmltitlerunning{RC-GRPO: Reward-Conditioned Group Relative Policy Optimization for Multi-Turn Tool Calling Agents}

\begin{document}

\twocolumn[
  \icmltitle{RC-GRPO: Reward-Conditioned Group Relative Policy Optimization for Multi-Turn Tool Calling Agents}



  \icmlsetsymbol{equal}{*}

  \begin{icmlauthorlist}
    \icmlauthor{Haitian Zhong}{casia,msra,zgca}
    \icmlauthor{Jixiu Zhai}{lzu,sii}
    \icmlauthor{Lei Song}{msra}
    \icmlauthor{Jiang Bian}{msra}
    \icmlauthor{Qiang Liu}{casia}
    \icmlauthor{Tieniu Tan}{casia,nju}
  \end{icmlauthorlist}

  \icmlaffiliation{casia}{New Laboratory
of Pattern Recognition (NLPR), State Key Laboratory of Multimodal Artificial Intelligence Systems (MAIS), Institute of Automation, Chinese Academy of Sciences}
  \icmlaffiliation{lzu}{School of Mathematics and Statistics, Lanzhou University}
  \icmlaffiliation{sii}{Shanghai Innovation Institute}
  \icmlaffiliation{msra}{Microsoft Research}
  \icmlaffiliation{nju}{Nanjing University}
  \icmlaffiliation{zgca}{Zhongguancun Academy}

  \icmlcorrespondingauthor{Lei Song}{lesong@microsoft.com}

  \icmlkeywords{Large Language Models, Reinforcement Learning, Tool Calling, Multi-turn Agents, Policy Optimization, Decision Transformer}

  \vskip 0.3in
]



\printAffiliationsAndNotice{}  
\begin{abstract}
Multi-turn tool calling is challenging for Large Language Models (LLMs) because rewards are sparse and exploration is expensive. A common recipe, SFT followed by GRPO, can stall when within-group reward variation is low (e.g., more rollouts in a group receive the all 0 or all 1 reward), making the group-normalized advantage uninformative and yielding vanishing updates. To address this problem, we propose \textbf{RC-GRPO} (Reward-Conditioned Group Relative Policy Optimization), which treats exploration as a controllable steering problem via discrete reward tokens. We first fine-tune a Reward-Conditioned Trajectory Policy (RCTP) on mixed-quality trajectories with reward goal special token (e.g., \texttt{<|high\_reward|>}, \texttt{<|low\_reward|>}) injected into the prompts, enabling the model to learn how to generate distinct quality trajectories on demand. Then during RL, we sample diverse reward tokens within each GRPO group and condition rollouts on the sampled token to improve within-group diversity, improving agvantage gains. On the Berkley Function Calling Leaderboard v4 (BFCLv4) multi-turn benchmark, our method yields consistantly improved performance than baselines, and the performance on Qwen-2.5-7B-Instruct even surpasses all closed-source API models.
\end{abstract}

\begin{figure*}[!t]
  \centering
  \includegraphics[width=\textwidth]{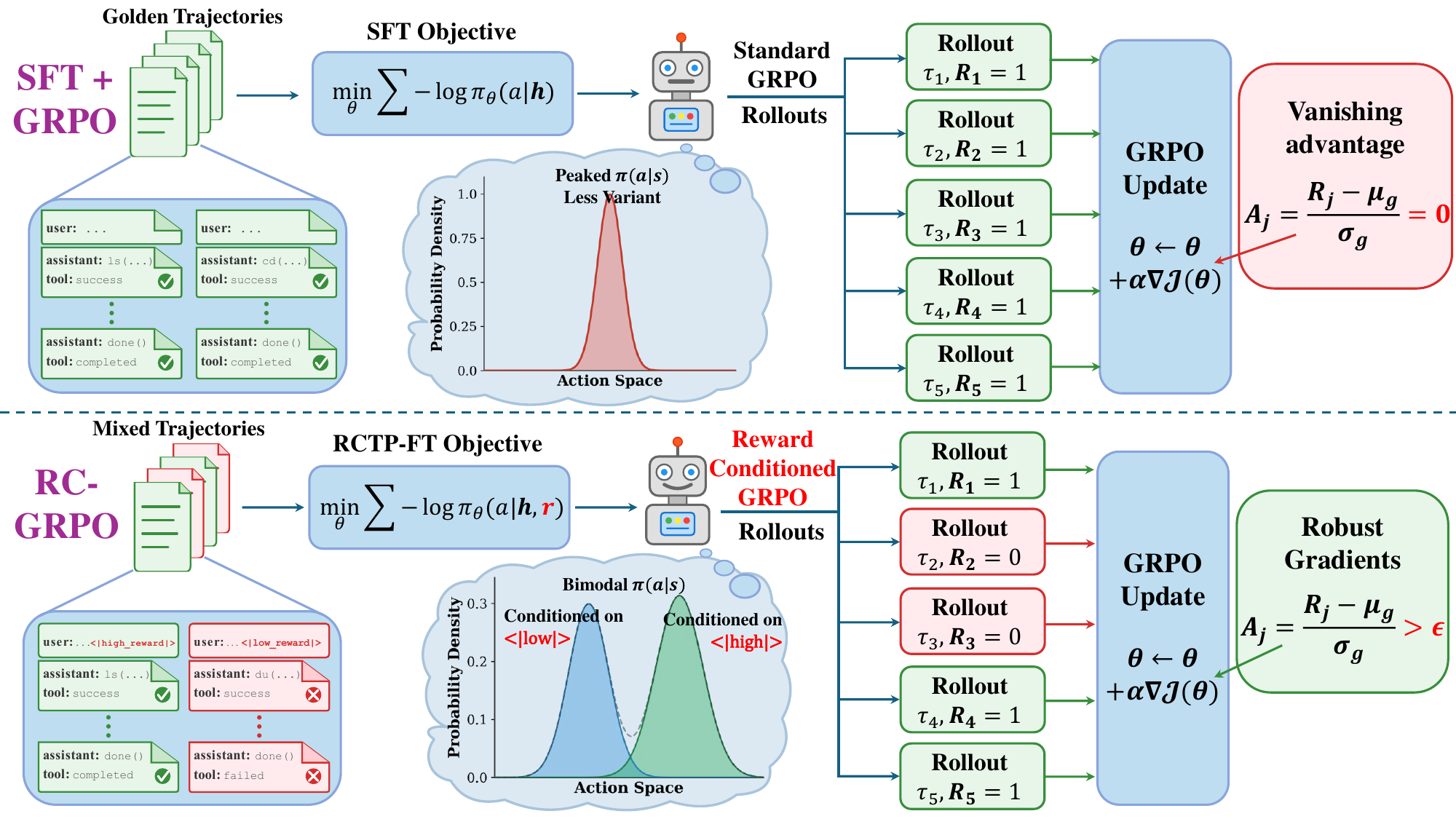}
  \caption{Overview of RC-GRPO. \textbf{(Top)} Standard GRPO optimization from an SFT initialization can sharply reduce rollout diversity, yielding low within-group reward variance and a weak/vanishing advantage signal. \textbf{(Bottom)} Our RC-GRPO conditions rollouts on discrete reward tokens and samples diverse tokens within each group, explicitly injecting variance and producing informative advantages.}
  \label{fig:overview}
\end{figure*}

\section{Introduction}
\label{sec:intro}

Tool-using Large Language Models (LLMs) can execute complex tasks by interleaving natural language reasoning with external API calls \cite{yao2023react,schick2023toolformer,gorilla2023,qin2023toolllm}. However, for multi-turn tool calling, success is often measured by sparse, trajectory-level rewards, making exploration costly. In this setting, group-relative policy-gradient methods such as Group Relative Policy Optimization (GRPO) are attractive because they avoid a critic, but their learning signal fundamentally depends on within-group variability: if the rewards within a sampled group have near-zero or even completely zero standard deviation, the group-normalized advantage becomes degenerate and policy updates vanish.

A practical failure mode then appears in the standard SFT-then-GRPO pipeline. SFT on optimal demonstrations intentionally produces a peaked policy (a strong ``golden-path'' prior), which reduces rollout diversity and can substantially reduce within-group reward variance under GRPO. This ``paradox of perfection'' is especially severe for multi-turn tool calling with partial observability \cite{kaelbling1998planning,sutton2018reinforcement}: a policy that is locally confident can repeatedly generate the same short-horizon behavior, leaving little informative contrast for group-relative advantages, consistent with recent analyses of vanishing updates when reward variability under the current policy is small \cite{razin2024vanishing}.

We propose \textbf{RC-GRPO} (Reward-Conditioned Group Relative Policy Optimization), which makes within-group diversity a controlled variable rather than a byproduct of sampling temperature. Inspired by return-conditioned generation \cite{chen2021decision,schmidhuber2019reinforcement}, we first train a Reward-Conditioned Trajectory Policy (RCTP) on mixed-quality trajectories so that the policy can reliably produce distinct behaviors under different tokens; then, during RL with GRPO, we sample diverse reward tokens within each group and condition rollouts on the sampled token, explicitly injecting variance and systematically restoring non-degenerate group-relative advantages.

Our principal contributions are summarized as follows.

\begin{itemize}
    \item We propose \textbf{RC-GRPO}. We first fine-tune a Reward-Conditioned Trajectory Policy (RCTP) on mixed-quality trajectories by appending a reward-goal special token to the prompt, so the policy learns to generate systematically different-quality trajectories under different tokens. We then train the model with reward-conditioned GRPO (RC-GRPO), where each GRPO group samples diverse reward tokens and conditions rollouts on them, ensuring non-degenerate within-group reward variation for group-normalized updates.
    \item We conduct a series of experiments on Berkley Function Calling Leaderboard v4 multi-turn tool calling split using LLaMA-3.1-8B-Instruct and Qwen2.5-7B-Instruct, and we report comparisons against standard SFT+GRPO baselines as well as several strong closed API models .
    \item We analyze training dynamics to justify whether improvements are attributable to increased randomness (entropy trajectories and entropy--reward correlation) and to quantify within-group learning signal quality (advantage spread with KL/gradient statistics). We further provide a variance-based theoretical analysis.
\end{itemize}

\section{Related Work}

\subsection{Tool-Calling LLMs and Benchmarks}
The capability of LLMs to use tools has been evaluated by benchmarks including the Berkeley Function Calling Leaderboard (BFCLv4) \cite{bfcl2024,gorilla2023}, API-Bank \cite{li2023apibank}, and ToolRL \cite{toolrl2024}. We focus on BFCLv4's multi-turn subset, which requires agents to maintain state across sequential tool calls to solve complex queries. In contrast, many tool-use benchmarks (e.g., ToolBench \cite{xu2023toolbench}) primarily evaluate single-turn function selection or parameter correctness without requiring persistent state tracking.

Toolformer \cite{schick2023toolformer} pioneered self-supervised tool learning, while ReAct \cite{yao2023react} introduced interleaved reasoning and acting. ToolLLM \cite{qin2023toolllm} scaled tool-use training to large collections of real-world APIs. While prompt engineering and standard SFT have shown promise, they often struggle with error recovery and long-horizon reasoning. Recent work focuses on closed-loop agents \cite{xi2023rise,wang2024executable}, which is closely related to our setting.

\subsection{RL-based Policy Optimization for LLMs}
Proximal Policy Optimization (PPO) \cite{schulman2017proximal} has been the standard RL algorithm for RLHF \cite{ouyang2022training,stiennon2020learning}. However, PPO requires a separate value network (critic), which doubles the memory footprint. GRPO \cite{deepseek2025r1} removes the critic by estimating advantages relative to a group of outputs generated from the same prompt, reducing memory by $\sim$50\%. Direct Preference Optimization (DPO) \cite{rafailov2023direct} similarly avoids explicit reward modeling but requires pairwise preferences.

Recent work has highlighted unique challenges in \emph{multi-turn} agent RL. RAGEN \cite{wang2025ragen} identifies the ``Echo Trap'' phenomenon where agents overfit to locally rewarded reasoning patterns, proposing trajectory-level rewards. SimpleTIR \cite{wang2025simpletir} addresses training instability from tool feedback deviating from pretrained distributions by filtering void turns during GRPO. Agentic RL with Implicit Step Rewards \cite{putta2025agentic} tackles sparse reward credit assignment via implicit process reward models. Agent Early Experience \cite{zhang2024agent} demonstrates that LLMs can perform model-based planning when fine-tuned to predict environment states.

\subsection{Return-Conditioned Learning}
Return-Conditioned Learning approaches, such as Decision Transformers \cite{chen2021decision} and Upside-Down RL \cite{schmidhuber2019reinforcement}, reframe reinforcement learning as a sequence modeling problem. Instead of estimating value functions or policy gradients, these methods learn a conditional policy $\pi(a|s, R)$ where $R$ represents the target return. During inference, the model is conditioned on a high return value to generate optimal trajectories. This paradigm has been extended to offline RL settings by the Trajectory Transformer \cite{janner2021offline}, which models the joint distribution of states, actions, and rewards.

\section{Method}
\label{sec:method}

Our method is motivated by a practical pathology observed when applying group-based policy optimization to strong tool-using LLMs. Starting from a high-quality SFT policy, rollouts within a GRPO group can become nearly identical, reducing intra-group reward variance and weakening the relative-advantage signal. In Sec.~\ref{sec:theory} (Q4 in Sec.~\ref{sec:experiments}), we formalize this ``gradient collapse'' phenomenon and show how reward-conditioned rollout generation restores within-group variance without requiring high policy entropy.

In Sec.~\ref{sec:preliminaries}, we first formalize multi-turn tool calling as a POMDP and introduce our reward-conditioned policy parameterization.
Sec.~\ref{sec:stage1} then describes Stage~1, where we fine-tune a Reward-Conditioned Trajectory Policy (RCTP) on mixed trajectories labeled by discrete reward tokens.
Sec.~\ref{sec:rcgrpo} presents Stage~2, where we optimize the policy using GRPO with reward-conditioned sampling to maintain within-group diversity.

\subsection{Preliminaries}
\label{sec:preliminaries}
We formalize multi-turn tool calling as a Partially Observable Markov Decision Process (POMDP). The agent interacts with the environment until episode termination. At each turn $t \in \{1, \ldots, T_{\text{max}}\}$, the user provides a natural language query $u_t$, and the agent generates an action $a_t$, a JSON-formatted tool call containing \texttt{name} and \texttt{args} fields. The environment executes this call and returns an observation $o_t$. The episode terminates when the agent invokes a termination action (e.g., \texttt{done()}) or reaches the maximum turn limit. The history $h_t = (u_0, a_0, o_0, \dots, u_{t-1}, a_{t-1}, o_{t-1}, u_t)$ accumulates all prior context up to the current query. A complete interaction forms a trajectory $\tau = (h_T, a_T, o_T)$.

Following the return-conditioned paradigm of Decision Transformers, we model the policy as $\pi_\theta(a_t | h_t, r)$, where $r \in \mathcal{R}$ is a discrete reward token indicating expected trajectory quality. The token set $\mathcal{R}$ contains two levels: \texttt{<|high\_reward|>} and \texttt{<|low\_reward|>}. This conditioning enables controllable generation: at inference, we set $r = \texttt{<|high\_reward|>}$ to elicit optimal behavior; during training, we sample diverse $r$ values to inject variance into rollout groups.

\subsection{Stage 1: Reward-Conditioned Trajectory Policy (RCTP) Finetuning}
\label{sec:stage1}
The first stage transforms the base LLM $\pi_{\text{base}}$ into a Reward-Conditioned Trajectory Policy (RCTP) capable of generating trajectories $\tau$ of varying quality based on the conditioning token $r$. We construct the mixed-quality dataset $\mathcal{D} = \{(\tau_i, r_i)\}_{i=1}^N$ by pairing each trajectory with its corresponding reward token. In practice, $\mathcal{D}$ is curated by combining expert (successful) trajectories from the benchmark ground truth with diverse failure trajectories generated by exploration rollouts, and then injecting the appropriate reward token into each example; see Appendix~\ref{app:dataset} for full details. \cite{bfcl2024}

\paragraph{Reward Token Quantization.}
We compute $R(\tau_i)$ for each trajectory using Eq.~\ref{eq:norm_reward} (defined in Sec.~\ref{sec:reward_function}). Given the binary nature of our reward signal, we map the outcomes directly to the categorical token set $\mathcal{R} = \{\texttt{<|high\_reward|>}, \texttt{<|low\_reward|>}\}$:
\begin{equation}
    r_{\text{token}}(R) = 
    \begin{cases} 
    \texttt{<|high\_reward|>} & \text{if } R = 1 \text{ (Success)} \\
    \texttt{<|low\_reward|>} & \text{if } R = 0 \text{ (Failure)}
    \end{cases}
\end{equation}

\paragraph{Finetuning Objective.}
We fine-tune $\pi_{\text{base}}$ to learn the conditional distribution $\pi_\theta(a_t | h_t, r)$. The reward token $r$ is prepended to the history $h_t$ before each assistant turn. The objective maximizes log-likelihood over $\mathcal{D}$:
\begin{equation}
    \mathcal{L}_{\text{RCTP}} = -\mathbb{E}_{(\tau,r)\sim\mathcal{D}} \left[ \sum_{t=1}^{T} \log \pi_\theta(a_t | h_t, r) \right]
\end{equation}
After training, we obtain the reference policy $\pi_{\text{ref}} \leftarrow \pi_\theta$, which serves as both the initialization and KL anchor for Stage 2. The model learns to correlate $r$ with action quality: conditioning on \texttt{<|high\_reward|>} produces optimal actions $a_t$, while \texttt{<|low\_reward|>} generates plausible failures.

\subsection{Stage 2: Reward-Conditioned Group Relative Policy Optimization (RC-GRPO)}
\label{sec:rcgrpo}

In Stage 2, we start from the Stage~1 reference policy $\pi_{\text{ref}}$ and optimize the trainable policy $\pi_\theta$ using GRPO with reward-conditioned sampling. For each prompt, GRPO draws a \emph{group} of $G$ trajectories; we obtain within-group diversity by sampling a discrete reward token $r \in \mathcal{R}$ for each trajectory from a fixed distribution $P_{\text{sample}}(r)$ (defined in Eq.~\ref{eq:sampling}) and conditioning rollouts on that token. This injects structured variance into each group---a prerequisite for non-degenerate, group-normalized advantages---and is not reliably achieved by temperature/entropy tuning alone when $\pi_{\text{ref}}$ is highly peaked.

\paragraph{Reward-Conditioned Rollout.}
For each prompt $u_0$, we generate a group of $G$ trajectories $\{\tau_1, \dots, \tau_G\}$. Unlike standard GRPO which samples from an unconditional $\pi_\theta(a|h)$, we first sample a reward token $r_j \sim P_{\text{sample}}(r)$ for each member $j \in \{1, \dots, G\}$:
\begin{equation}
\label{eq:sampling}
    P_{\text{sample}}(r) = 
    \begin{cases} 
    p & \text{if } r = \texttt{<|high\_reward|>} \\
    1-p & \text{if } r = \texttt{<|low\_reward|>}
    \end{cases}
\end{equation}
Here, $p$ is a hyperparameter set to match the proportion of successful (expert) trajectories in the RCTP's training dataset $\mathcal{D}$; the setting of this parameter can be found in Appendix~\ref{app:bfcl_curation}. This ensures that the sampling prior aligns with the model's learned distribution, preventing distribution shift while still guaranteeing variance injection. Then, for each turn $t$ in trajectory $\tau_j$, the policy generates:
\begin{equation}
    a_{t,j} \sim \pi_\theta(\cdot | h_{t,j}, r_j)
\end{equation}
This steers generation toward diverse quality modes: trajectories with $r_j = \texttt{<|high\_reward|>}$ tend toward optimal actions, while those with $r_j = \texttt{<|low\_reward|>}$ explore suboptimal paths—preventing variance collapse.

\paragraph{Trajectory-Level Reward Function.}
\label{sec:reward_function}
To align with the unified evaluation framework of modern agentic benchmarks, we adopt a unified trajectory-level reward function $R(\tau)$ that evaluates the overall success of the interaction. This reward serves dual purposes: (1) constructing the reward tokens $r \in \mathcal{R}$ for building dataset $\mathcal{D}$ in Stage 1, and (2) calculating advantages during GRPO in Stage 2.

We formulate the reward as a composition of two essential factors:
\begin{equation}
\label{eq:norm_reward}
    R(\tau) = R_{\text{state}} \cdot R_{\text{action}}
\end{equation}

\paragraph{State/Goal Consistency ($R_{\text{state}}$).}
Measures whether the agent drives the environment to a correct terminal condition. For BFCL, this compares the final environment state against the state obtained by replaying golden actions (e.g., exact match of file systems or databases).

\paragraph{Essential Action / Constraint Coverage ($R_{\text{action}}$).}
Measures whether required actions and constraints are satisfied over the full interaction. For BFCL, we check that every essential tool call from the ground truth appears in the trajectory with correct parameters.

In this framework, the reward acts as a binary success signal ($R(\tau) \in \{0, 1\}$). A trajectory is considered successful ($R=1$) only if it achieves the desired state \emph{and} satisfies all procedural constraints; otherwise, it is failure ($R=0$). This binary signal maps directly to our token set: $R=1 \implies \texttt{<|high\_reward|>}$ and $R=0 \implies \texttt{<|low\_reward|>}$.

\paragraph{Group-Relative Advantage.}We compute $R(\tau_j)$ for each trajectory using Eq.~\ref{eq:norm_reward}. Following GRPO \cite{deepseek2025r1}, we normalize rewards within the group to obtain advantage $A_j$. Because our group contains trajectories under different tokens $r_j \in \mathcal{R}$, the group statistics capture variance across quality modes:
\begin{equation}
\begin{split}
    A_j = \frac{R(\tau_j) - \mu_g}{\sigma_g + \epsilon_{\text{stab}}}, \quad \text{where } & \mu_g = \frac{1}{G}\sum_{k=1}^G R(\tau_k), \\
    & \sigma_g = \sqrt{\frac{1}{G}\sum_{k=1}^G (R(\tau_k) - \mu_g)^2}
\end{split}
\end{equation}
Here $\epsilon_{\text{stab}} > 0$ is a small numerical stability constant.

\paragraph{Optimization Objective.}
The RC-GRPO loss follows the PPO-style clipped objective used in GRPO \cite{deepseek2025r1}, adapted to the conditioned policy $\pi_\theta(a_t | h_t, r)$.
Let $\pi_{\theta_{\text{old}}}$ denote the policy that generated the sampled group, and define the (trajectory-level) importance ratio
\begin{equation}
\rho_j(\theta) = \prod_{t=1}^{T} \frac{\pi_\theta(a_{t,j} | h_{t,j}, r_j)}{\pi_{\theta_{\text{old}}}(a_{t,j} | h_{t,j}, r_j)}.
\end{equation}
Then the loss is
\begin{equation}
\label{eq:rcgrpo_loss}
\begin{aligned}
\mathcal{L}_{\text{RC-GRPO}}(\theta)
&= -\mathbb{E}_{u_0 \sim \mathcal{D}_{\text{train}}}
\Bigg[
\frac{1}{G}\sum_{j=1}^{G} \ell^{\text{clip}}_j(\theta) \\
&\quad\; - \beta\, D_{\text{KL}}\!\left(\pi_\theta(\cdot\mid h,r) \,\|\, \pi_{\text{ref}}(\cdot\mid h,r)\right)
\Bigg]
\end{aligned}
\end{equation}
where
\begin{equation}
\ell^{\text{clip}}_j(\theta)
= \min\!\left(
\rho_j(\theta) A_j,\;
\operatorname{clip}\big(\rho_j(\theta), 1-\epsilon, 1+\epsilon\big)\,A_j
\right).
\end{equation}
Here $\epsilon$ is the clipping range and $\beta$ is the KL coefficient. Over training, $\pi_\theta$ improves the expected $R(\tau)$ across sampled conditions while the clipping and KL regularization stabilize updates.

\begin{algorithm}[t]
  \caption{RC-GRPO Full Pipeline}
  \label{alg:rcgrpo}
  \begin{algorithmic}[1]
    \STATE {\bfseries Require:} Policy $\pi_\theta$, reference policy $\pi_{\text{ref}} \leftarrow \pi_{\text{base}}$
    \STATE {\bfseries Require:} Dataset $\mathcal{D}$, training prompts $\mathcal{D}_{\text{train}}$
    \STATE {\bfseries Require:} Hyperparameters: learning rate $\eta$, KL coefficient $\beta$, group size $G$, sampling distribution $P_{\text{sample}}(r)$
    \STATE {\bfseries Notation:} $\tau = (a_1, \ldots, a_T)$ trajectory; $h_t$ history at step $t$; $\rho_j(\theta) = \pi_\theta(\tau_j) / \pi_{\text{ref}}(\tau_j)$ importance ratio

    \STATE \textcolor{blue}{// STAGE 1: REWARD-CONDITIONED TRAJECTORY POLICY (RCTP) FINETUNING}
    \STATE Compute $R(\tau)$ for each trajectory $\tau \in \mathcal{D}$ (state match $\wedge$ essential-action coverage; see Sec.~\ref{sec:reward_function} and App.~\ref{app:reward_details}) \hfill $\triangleright$ Eq.~\ref{eq:norm_reward}
    \STATE Assign a reward token $r(\tau) \in \mathcal{R}$ via $r(\tau)=\texttt{<|high\_reward|>}$ if $R(\tau)=1$, else $\texttt{<|low\_reward|>}$ \hfill $\triangleright$ Sec.~\ref{sec:stage1}
    \STATE Define $\mathcal{L}_{\text{RCTP}} = -\mathbb{E}_{(\tau,r)\sim\mathcal{D}}\left[\sum_{t=1}^{T}\log \pi_{\text{ref}}(a_t\mid h_t, r)\right]$
    \STATE Update $\pi_{\text{ref}}$ by minimizing $\mathcal{L}_{\text{RCTP}}$

    \STATE \textcolor{blue}{// STAGE 2: REWARD-CONDITIONED GRPO (RC-GRPO)}
    \STATE Initialize $\pi_\theta \leftarrow \pi_{\text{ref}}$
    \REPEAT
    \FOR{each prompt $u_0 \in \mathcal{D}_{\text{train}}$}
      \FOR{$j = 1$ {\bfseries to} $G$}
        \STATE Sample $r_j \sim P_{\text{sample}}(r)$ and roll out $\tau_j \sim \pi_\theta(\cdot\mid h, r_j)$
      \ENDFOR
      \STATE Compute rewards $\{R(\tau_j)\}_{j=1}^{G}$ and group-normalized advantages $\{A_j\}_{j=1}^{G}$
      \STATE Update $\theta \leftarrow \theta - \eta\,\nabla_\theta\mathcal{L}_{\text{RC-GRPO}}(\theta)$, where
      \STATE \hspace{1em}\parbox[t]{0.9\linewidth}{$\mathcal{L}_{\text{RC-GRPO}}(\theta)= -\frac{1}{G}\sum_{j=1}^{G}\min\big(\rho_j(\theta)A_j,\;\mathrm{clip}(\rho_j(\theta),1-\epsilon,1+\epsilon)A_j\big) + \beta\,D_{\mathrm{KL}}(\pi_\theta\,\|\,\pi_{\text{ref}})$.}
    \ENDFOR
    \UNTIL{convergence}
  \end{algorithmic}
\end{algorithm}

\section{Experiments}
\label{sec:experiments}

\subsection{Experimental Setup}
We evaluate our method on \textbf{Berkeley Function Calling Leaderboard (BFCLv4)} \cite{bfcl2024} using two base models: \textbf{LLaMA-3.1-8B-Instruct} \cite{dubey2024llama3} and \textbf{Qwen2.5-7B-Instruct} \cite{yang2024qwen2}.

\paragraph{Compared Methods.}
We compare five training configurations. The \emph{Base Model} is the original instruction-tuned LLM without any additional task-specific finetuning on BFCLv4. Our three baselines are: (i) supervised finetuning (SFT) followed by Group Relative Policy Optimization (GRPO), (ii) SFT followed by reward-conditioned GRPO (RC-GRPO), and (iii) reward-conditioned trajectory-policy finetuning (RCTP-FT) followed by standard GRPO. Our full method combines RCTP-FT initialization with RC-GRPO in the RL stage.

\paragraph{API Model Baselines.}
In addition, we report BFCLv4 validation accuracy for several closed API models (Opus-4.5, Sonnet-4.5, GLM-4.7, Gemini-3-Pro, GPT-5.2) evaluated under the same API-calling setting.

At inference, we condition on $r = \texttt{<|high\_reward|>}$ for optimal performance. Hyperparameter configurations are provided in Appendix~\ref{app:experiment_settings}.


\subsection{Main Results}
We first present the overall performance of RC-GRPO compared to the baselines. Table~\ref{tab:main_results} reports accuracy on BFCLv4 (overall and by category).

\begin{table*}[t]
\centering
\small
\setlength{\tabcolsep}{4pt}
\renewcommand{\arraystretch}{1.08}
\begin{tabular}{@{}llccccc@{}}
\toprule
\textbf{Model} & \textbf{Method/Version} & \textbf{Overall Acc.} & \textbf{Base} & \textbf{Miss Func} & \textbf{Miss Param} & \textbf{Long Context} \\
\midrule
\multirow{5}{*}{\cellcolor{white}\textbf{Qwen2.5-7B}} & Base Model & 11.25\% & 11.11\% & 17.65\% & 4.54\% & 13.04\% \\
& SFT + GRPO & 48.75\% & 55.56\% & 35.29\% & 50.00\% &  52.17\%\\
& SFT + RC-GRPO & 46.25\% & 50.00\% & 35.29\% & 50.00\% & 47.82\% \\
& RCTP-FT + GRPO & \underline{73.75\%} & \underline{66.67\%} & \underline{64.71\%} & \underline{77.27\%} & \underline{82.61\%} \\
\rowcolor{blue!10}
& RCTP-FT + RC-GRPO (Ours) & \textbf{85.00\%} & \textbf{77.78\%} & \textbf{88.23\%}  & \textbf{81.81\%} & \textbf{91.30\%} \\
\midrule
\multirow{5}{*}{\cellcolor{white}\textbf{LLaMA-3.1-8B}} & Base Model & 0.00\% & 0.00\% & 0.00\% & 0.00\% & 0.00\% \\
& SFT + GRPO & 35.00\% & 38.89\% & 17.67\% & 40.91\% & 39.13\% \\
& SFT + RC-GRPO & 35.00\% & 44.44\% & 17.64\% & 36.36\% & 39.13\% \\
& RCTP-FT + GRPO & \underline{46.25\%} & \underline{33.33\%} & \textbf{35.29\%} & \underline{54.55\%} & \textbf{56.52\%} \\
\rowcolor{blue!10}
& RCTP-FT + RC-GRPO (Ours) & \textbf{48.75\%} & \textbf{38.89\%} & \textbf{35.29\%} & \textbf{60.87\%} & \underline{54.54\%} \\
\midrule
\textbf{Opus-4.5} & {\footnotesize claude-opus-4-5-20251101} & 61.25\% & 55.56\% & 65.22\% & 64.71\% & 59.09\% \\
\midrule
\textbf{GLM-4.7} & {\footnotesize glm-4.7} & 58.75\% & 55.56\% & 58.82\% & 54.55\% & 65.22\% \\
\midrule
\textbf{Sonnet-4.5} & {\footnotesize claude-sonnet-4-5-20250929} & 57.50\% & 44.44\% & 58.82\% & 59.09\% & 65.22\% \\
\midrule
\textbf{Gemini-3-Pro} & {\footnotesize gemini-3-pro-preview} & 53.75\% & 55.56\% & 64.71\% & 36.36\% & 60.87\% \\
\midrule
\textbf{GPT-5.2} & {\footnotesize gpt-5.2-2025-12-11} & 50.00\% & 38.89\% & 47.06\% & 45.45\% & 65.22\% \\
\bottomrule
\end{tabular}
\caption{Performance comparison on the BFCLv4 validation split (overall and by category).}
\label{tab:main_results}
\end{table*}

On \textbf{Qwen2.5-7B}, our full pipeline (RCTP-FT + RC-GRPO) achieves \textbf{85.00\%} overall accuracy, improving upon all reported baselines: Base Model (11.25\%), SFT + GRPO (48.75\%), SFT + RC-GRPO (46.25\%), and RCTP-FT + GRPO (73.75\%).
On \textbf{LLaMA-3.1-8B}, our method achieves \textbf{48.75\%} overall accuracy, improving over SFT + GRPO (35.00\%) and matching SFT + RC-GRPO (35.00\%), and outperforming the Base Model (0.00\%) and RCTP-FT + GRPO (46.25\%).
For reference, the best-performing closed API baseline in Table~\ref{tab:main_results} achieves 61.25\% overall accuracy, which remains below our best open-weights result (85.00\% on Qwen2.5-7B).
Full model/version details are summarized in Appendix~\ref{app:model_details}.

\subsection{Analysis}
To better understand where RC-GRPO helps (and why), we break the analysis into four focused questions, covering the roles of reward conditioning, RCTP initialization, training dynamics, and a supporting theoretical explanation.

\textbf{Q1: Does Reward Conditioning improve multi-turn tool calling ability compared to traditional GRPO?} \\
To isolate the effect of Reward Conditioning (RC) during RL, we compare traditional GRPO vs.\ RC-GRPO while holding the initialization fixed to RCTP-FT in Table~\ref{tab:ablation_rc}.

\begin{table}[ht]
\centering
\caption{Ablation: Effect of Reward Conditioning (RC) under RCTP-FT initialization.}
\label{tab:ablation_rc}
\small
\begin{tabular*}{\columnwidth}{@{\extracolsep{\fill}}lccc}
\toprule
\textbf{Setting} & \textbf{w/o RC} & \textbf{w/ RC} & \textbf{$\Delta$} \\
\midrule
\multicolumn{4}{l}{\textit{\textbf{LLaMA-3.1-8B}}} \\
RCTP-FT init & 46.25\% & 48.75\% & +2.50\% \\
\midrule
\multicolumn{4}{l}{\textit{\textbf{Qwen2.5-7B}}} \\
RCTP-FT init & 73.75\% & 85.00\% & +11.25\% \\
\bottomrule
\end{tabular*}
\end{table}

The results show that, once the policy is initialized with RCTP-FT, adding RC during RL consistently improves performance: +2.50\% on LLaMA-3.1-8B (46.25\%$\to$48.75\%) and +11.25\% on Qwen2.5-7B (73.75\%$\to$85.00\%).

\textbf{Q2: Is the RCTP initialization necessary?} \\
Table~\ref{tab:ablation_rctp_init} summarizes two complementary observations: (i) starting from an SFT initialization, adding reward conditioning during RL has little impact, and (ii) under RC-GRPO, switching the initialization from traditional FT to RCTP-FT yields a large gain.

\begin{table}[ht]
\centering
\caption{RCTP-FT is necessary: (A) from an SFT init, adding RC during RL has little impact; (B) under RC-GRPO, switching to RCTP-FT yields large gains.}
\label{tab:ablation_rctp_init}
\small
\begin{tabular*}{\columnwidth}{@{\extracolsep{\fill}}lccc}
\toprule
\multicolumn{4}{l}{\textbf{(A) Add RC during RL (SFT init)}} \\
\textbf{Model} & \textbf{w/o RC} & \textbf{w/ RC} & \textbf{$\Delta$} \\
\midrule
LLaMA-3.1-8B & 35.00\% & 35.00\% & +0.00\% \\
Qwen2.5-7B & 48.75\% & 46.25\% & -2.50\% \\
\midrule
\multicolumn{4}{l}{\textbf{(B) Switch init under RC-GRPO}} \\
\textbf{Model} & \textbf{w/o RCTP-FT} & \textbf{w/ RCTP-FT} & \textbf{$\Delta$} \\
\midrule
LLaMA-3.1-8B & 35.00\% & 48.75\% & +13.75\% \\
Qwen2.5-7B & 46.25\% & 85.00\% & +38.75\% \\
\bottomrule
\end{tabular*}
\end{table}

Table~\ref{tab:ablation_rctp_init}(A) shows that adding RC during RL from an SFT initialization yields negligible gains, likely because SFT does not expose the policy to mixed-quality, reward-conditioned trajectories, which highlights the necessity of RCTP-FT. In contrast, Table~\ref{tab:ablation_rctp_init}(B) shows that switching the initialization from SFT to RCTP-FT under RC-GRPO yields large improvements: +13.75\% on LLaMA-3.1-8B (35.00\%$\to$48.75\%) and +38.75\% on Qwen2.5-7B (46.25\%$\to$85.00\%).

\textbf{Q3: Does reward conditioning improve GRPO by stabilizing the update signal (non-vanishing KL/advantage spread), rather than by increasing randomness?} \\
\label{q3}
A natural concern is that reward conditioning might act like a proxy for higher sampling temperature (or an entropy bonus), i.e., improving exploration primarily by increasing randomness.
To test this hypothesis, we analyze training logs from four Qwen2.5-7B BFCLv4 runs and examine (i) the entropy trajectory and (ii) the correlation between entropy and training reward within each run.
Here, $\rho$ denotes the Pearson correlation coefficient computed across training steps between the logged entropy values and the logged training reward values (no smoothing).
We compute early/late summaries using 70-step windows, since 70 is approximately the last one-fifth of the total training horizon in these runs.

\begin{table}[t]
\centering
\caption{Entropy trajectory and entropy--reward correlation. RC-GRPO achieves improvements with \emph{decreasing} entropy and a negative entropy--reward correlation, suggesting its benefit is not explained by higher randomness.}
\label{tab:q3_entropy}
\scriptsize
\setlength{\tabcolsep}{3pt}
\renewcommand{\arraystretch}{1.05}
\begin{tabular*}{\columnwidth}{@{\extracolsep{\fill}}lccc}
\toprule
\textbf{Method} & \textbf{Entropy} & \textbf{$\Delta$} & \textbf{$\rho$ (entropy, reward)} \\
\midrule
\rowcolor{blue!10}
RCTP+RC (Ours) & 0.079$\to$0.037 & $\downarrow$54\% & $-0.15$** \\
RCTP+GRPO & 0.029$\to$0.066 & $\uparrow$127\% & $+0.49$*** \\
SFT+GRPO & 0.037$\to$0.216 & $\uparrow$489\% & $+0.06$ (n.s.) \\
SFT+RC & 0.104$\to$0.317 & $\uparrow$205\% & $+0.09$ (n.s.) \\
\bottomrule
\end{tabular*}
\vspace{0.2em}

\footnotesize{Early and late averages use 70-step windows (roughly one-fifth of training): Early = steps 1--70, Late = steps 280--350. ***: $p<0.001$, **: $p<0.01$, n.s.: not significant.}
\end{table}

\begin{figure*}[t]
  \centering
  \includegraphics[width=\textwidth]{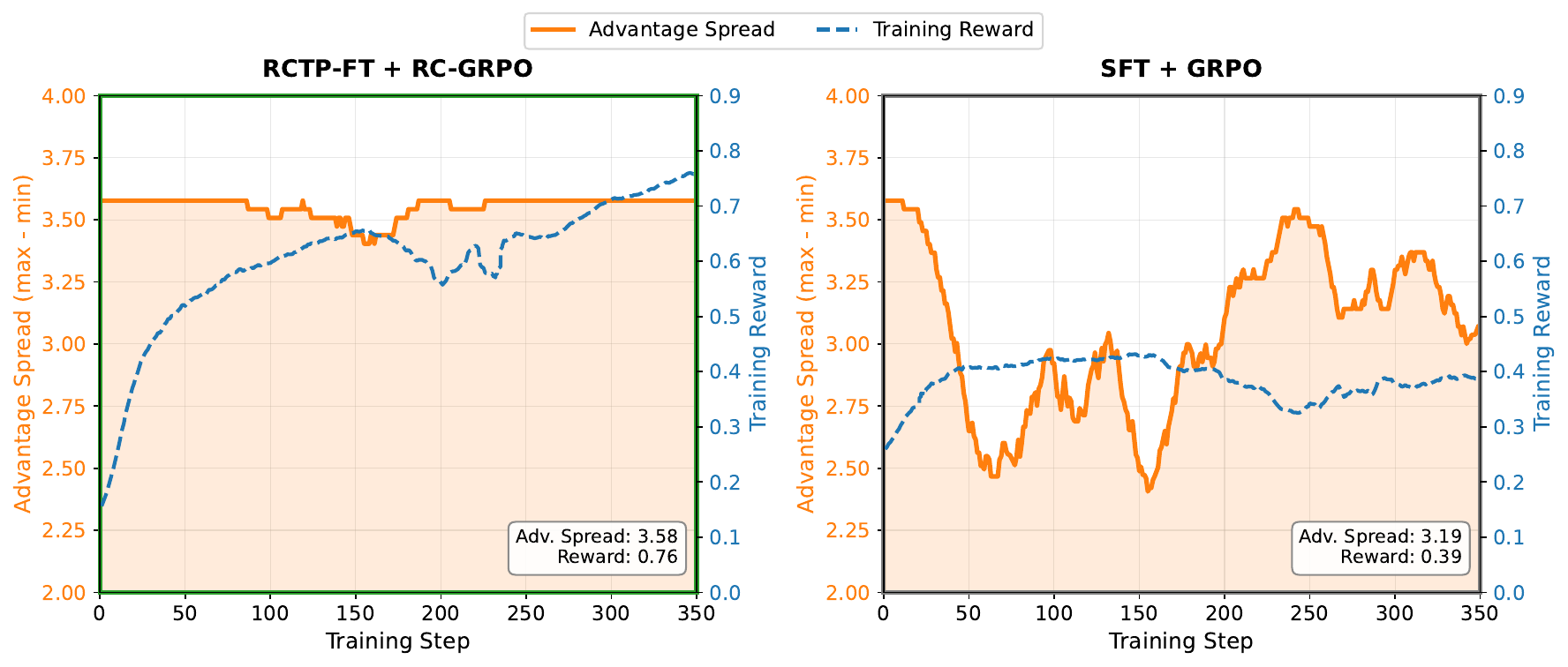}
  \caption{Training dynamics for Qwen2.5-7B on BFCLv4. We plot a proxy for within-group diversity (the gap between the maximum and minimum advantage within each sampled group) together with the training reward over time.}
  \label{fig:q3_advspread}
\end{figure*}

Taken together, Table~\ref{tab:q3_entropy} and Fig.~\ref{fig:q3_advspread} show that higher entropy is not required to maintain within-group diversity under RC-GRPO.
In our best-performing setting (RCTP+RC), entropy decreases over training (0.079$\to$0.037), yet reward improves and the entropy--reward correlation is negative ($\rho=-0.15$).
Meanwhile, under standard GRPO from the same RCTP initialization (RCTP+GRPO), reward is strongly positively correlated with entropy ($\rho=+0.49$), suggesting that entropy-based knobs act as an indirect and brittle route to exploration when the policy becomes peaked.

\textbf{Q4: Can we conduct theoretical justifications to explain why RC-GRPO works?} \\
\label{q4}
So far, we have focused on empirical results. We now complement them with a simple theoretical analysis and connect it to the observed training dynamics, to better illustrate why RC-GRPO is stable.
Specifically, we first analyze when standard GRPO can suffer from weak/vanishing group-normalized advantages, and then explain how reward-conditioned sampling restores within-group reward variance.

\label{sec:theory}
We do not claim a complete theory of multi-turn agent training dynamics; rather, we propose a minimal variance-based explanation for a common failure mode of standard GRPO when initialized from a peaked policy (e.g., after strong SFT). The key observation is that the Stage~2 GRPO update is mediated by the group-normalized advantage
$A_j = (R(\tau_j)-\mu_g)/(\sigma_g+\epsilon_{\text{stab}})$.
When group rollouts receive identical rewards, we have $R(\tau_j)=\mu_g$ for all $j$, so the numerator $(R(\tau_j)-\mu_g)$ is exactly zero and therefore $A_j=0$ (regardless of $\sigma_g$ and $\epsilon_{\text{stab}}$). In this case, the advantage-weighted update vanishes and learning stalls.
More generally, when rewards within a group are nearly identical, both $(R(\tau_j)-\mu_g)$ and $\sigma_g$ become small, yielding near-zero and/or noisy advantages (often effectively limited by the stability term $\epsilon_{\text{stab}}$), which again weakens the learning signal.
Proposition~\ref{prop:collapse} formalizes one sufficient condition under which such collapse can arise for standard (unconditioned) GRPO on peaked policies, while Proposition~\ref{prop:variance} shows how reward-conditioned sampling can prevent it by injecting between-mode reward variability within each group.

\begin{definition}[GRPO Advantage Collapse]
In GRPO with group size $G$, the advantage for trajectory $\tau_j$ is:
\begin{equation}
A_j = \frac{R(\tau_j) - \mu_g}{\sigma_g + \epsilon_{\text{stab}}}
\end{equation}
where $\mu_g$ and $\sigma_g$ are the within-group mean and standard deviation of $R(\tau)$. The advantage \emph{collapses} when $\sigma_g \to 0$, causing $A_j \to 0$ for all $j$ regardless of actual rewards.
\end{definition}

\begin{proposition}[Vanishing Gradient in Peaked Policies]
\label{prop:collapse}
Let $\pi_{\text{ref}}$ be trained on optimal demonstrations. Suppose that for each step $t$ (and history $h_t$ on the optimal trajectory), the SFT objective achieves a small per-step cross-entropy/KL to the optimal Dirac policy $\pi^*(\cdot|h_t)$:
\begin{equation}
    D_{\text{KL}}(\pi^*(\cdot|h_t)\,\|\,\pi_{\text{ref}}(\cdot|h_t)) \le \epsilon_{\text{sft}} .
\end{equation}
Then for a group of $G$ independent trajectories $\{\tau_1, \dots, \tau_G\}$ sampled from $\pi_{\text{ref}}$, the probability that all trajectories match the optimal trajectory $\tau^*$ satisfies
\begin{equation}
P(\tau_1 = \tau_2 = \cdots = \tau_G = \tau^*) \;\ge\; e^{-GT\epsilon_{\text{sft}}} \;\ge\; 1 - GT\epsilon_{\text{sft}}.
\end{equation}
On the event $\{\tau_1 = \cdots = \tau_G = \tau^*\}$, the within-group rewards are identical, so $\sigma_g = 0$ and $A_j = 0$ for all $j$. More generally, when $\pi_{\text{ref}}$ is sufficiently peaked so that rollouts within a group induce nearly identical rewards, we have $\sigma_g \ll \epsilon_{\text{stab}}$ and the normalized advantages are dominated by $\epsilon_{\text{stab}}$, making the effective learning signal negligible.
\end{proposition}

\begin{proof}[Proof Sketch]
Standard SFT on optimal-only data $\mathcal{D}_{\text{opt}}$ minimizes the negative log-likelihood of the optimal action, which can be written as $D_{\text{KL}}(\pi^* \| \pi_{\text{ref}})$ when $\pi^*$ is a Dirac measure on the optimal action. If $D_{\text{KL}}(\pi^*(\cdot|h_t)\,\|\,\pi_{\text{ref}}(\cdot|h_t)) \le \epsilon_{\text{sft}}$, then $-\log \pi_{\text{ref}}(a_t^*|h_t) \le \epsilon_{\text{sft}}$, hence $\pi_{\text{ref}}(a_t^*|h_t) \ge e^{-\epsilon_{\text{sft}}} \ge 1-\epsilon_{\text{sft}}$. Over $T$ steps, $P(\tau=\tau^*) \ge e^{-T\epsilon_{\text{sft}}} \ge 1-T\epsilon_{\text{sft}}$, and for $G$ independent samples, $P(\tau_1=\cdots=\tau_G=\tau^*) \ge e^{-GT\epsilon_{\text{sft}}} \ge 1-GT\epsilon_{\text{sft}}$. When trajectories are identical, $R(\tau_1)=\cdots=R(\tau_G)$ so $\sigma_g=0$ and all advantages collapse (hence the advantage-weighted policy-gradient term is zero). See Appendix~\ref{app:proofs} for the full proof.
\end{proof}

The key insight of RC-GRPO is to \emph{inject variance through conditioning} via the reward token $r \in \mathcal{R}$:

\begin{proposition}[Variance Guarantee via Reward Conditioning]
\label{prop:variance}
In RC-GRPO, each trajectory $\tau_j$ is conditioned on token $r_j \sim P_{\text{sample}}(r)$. If $\pi_{\text{ref}}$ learns distinct modes such that $|\mathbb{E}[R(\tau)|r_i] - \mathbb{E}[R(\tau)|r_j]| \geq \epsilon$ for $r_i \neq r_j$, then the within-group variance is lower-bounded:
\begin{equation}
\mathbb{E}\left[\sigma_g^2\right] \geq \kappa \epsilon^2
\end{equation}
where $\sigma_g^2 = \frac{1}{G}\sum_{j=1}^G (R(\tau_j)-\mu_g)^2$ is the (biased) within-group second central moment and $\kappa > 0$ depends on $P_{\text{sample}}(r)$ and $G$. In particular, with $p$ defined as in Eq.~\ref{eq:sampling}, one may take $\kappa=\frac{G-1}{G}p(1-p)$.
\end{proposition}

\begin{proof}[Proof Sketch]
By the law of total variance, $\text{Var}(R(\tau)) \ge \text{Var}_{r}(\mathbb{E}[R(\tau)|r])$. If the conditional means are separated by at least $\epsilon$, then $\text{Var}(R(\tau))\ge p(1-p)\epsilon^2$. Since $\sigma_g^2$ is the group second central moment over $G$ i.i.d.\ draws, $\mathbb{E}[\sigma_g^2]=\frac{G-1}{G}\text{Var}(R(\tau)) \ge \frac{G-1}{G}p(1-p)\epsilon^2$. See Appendix~\ref{app:proofs} for details.
\end{proof}

\begin{remark}[Convergence Implication]
\label{rem:convergence}
The expected variance lower bound $\mathbb{E}[\sigma_g^2] \geq \kappa\epsilon^2$ has direct implications for convergence speed. The signal-to-noise ratio of the GRPO gradient estimator depends on the typical scale of $\sigma_g$; by enforcing between-mode variability through $r$ and thereby preventing collapse of $\sigma_g$ in expectation, RC-GRPO maintains informative group-relative advantages throughout training. This contrasts with standard GRPO on peaked SFT models, where groups concentrate on a single trajectory and $\sigma_g$ collapses.
\end{remark}

In combination, Proposition~\ref{prop:collapse} and Proposition~\ref{prop:variance} explain why RC-GRPO avoids within-group variance collapse in practice. To empirically validate these theoretical claims, we analyze training dynamics in the late phase (the last 70 training steps) across four Qwen2.5-7B runs. Guided by Proposition~\ref{prop:collapse} and Proposition~\ref{prop:variance}, we operationalize within-group diversity using the \emph{advantage spread}---the difference between the maximum and minimum group-normalized advantage within each sampled group---as a direct proxy for how much signal the group normalization provides.

\begin{table}[t]
\centering
\caption{Late-phase training dynamics (last 70 training steps). RC-GRPO maintains high within-group diversity while keeping policy entropy low, consistent with the variance-guarantee mechanism.}
\label{tab:q4_dynamics}
\scriptsize
\setlength{\tabcolsep}{2pt}
\renewcommand{\arraystretch}{1.05}
\begin{tabular*}{\columnwidth}{@{\extracolsep{\fill}}lcccc}
\toprule
\textbf{Method} & \textbf{Spread} & \textbf{KL} & \textbf{$\|g\|$} & \textbf{$H$} \\
\midrule
\rowcolor{blue!10}
RCTP+RC (Ours) & \textbf{3.58} & 0.0010 & 2.7 & \textbf{0.037} \\
RCTP+GRPO & 3.46 & 0.0004 & 2.5 & 0.066 \\
SFT+GRPO & 3.22 & 0.0040 & 6.5 & 0.216 \\
SFT+RC & 3.51 & 0.0033 & 14.2 & 0.317 \\
\bottomrule
\end{tabular*}
\vspace{0.2em}

\footnotesize{Spread: max--min advantage within a sampled group; $\|g\|$: gradient norm; $H$: entropy.}
\end{table}

Table~\ref{tab:q4_dynamics} supports the variance-guarantee interpretation. Our method achieves the largest advantage spread (3.58) while maintaining the lowest entropy (0.037), indicating that within-group diversity is preserved without relying on increased policy randomness. In contrast, SFT + GRPO exhibits substantially higher entropy yet lower advantage spread, suggesting that entropy-based exploration alone is insufficient to reliably restore the within-group variability needed for informative group-normalized advantages.
Across all runs, gradient norms remain non-zero, but the combination of high spread and low entropy in RC-GRPO is most consistent with the intended between-mode separation induced by reward conditioning.

\section{Conclusion}
We introduced RC-GRPO, a two-stage pipeline (RCTP finetuning + reward-conditioned GRPO) for multi-turn tool calling that mitigates variance collapse in group-normalized policy optimization by making within-group diversity a controlled variable. Empirically, RC-GRPO delivers consistent gains on BFCLv4 for both Qwen2.5-7B and LLaMA-3.1-8B, with particularly large improvements when combined with the RCTP initialization. Our analysis suggests these gains are explained by more informative group-relative advantages (and non-vanishing updates) rather than simply increasing policy randomness. Theoretically, we provide a variance-based perspective that clarifies when standard GRPO can degenerate under peaked SFT policies and why reward-conditioned sampling prevents collapse. We hope this framing helps bridge controllable generation and group-based RL for training more reliable multi-turn tool-using agents.

\section*{Impact Statement}
This paper presents work whose goal is to advance the field of machine learning by improving the stability of reinforcement-learning optimization for multi-turn tool-calling agents. More reliable tool use can support beneficial applications such as automating routine digital tasks and enabling more capable assistants; at the same time, like other tool-using LLM systems, such capabilities may have dual-use implications if deployed without appropriate operational safeguards (e.g., access controls and monitoring). We expect the primary societal consequence of this work to be enabling further research on robust agent training and evaluation, rather than introducing fundamentally new deployment risks beyond those already associated with tool-enabled language models.

\bibliography{example_paper}
\bibliographystyle{icml2026}

\newpage
\appendix
\onecolumn
\section{Theoretical Proofs}
\label{app:proofs}

This appendix provides formal proofs for the theoretical results presented in the main text.

\subsection{Proof of Proposition~\ref{prop:collapse}}

We first state the definitions and lemmas used in the proof.

\begin{definition}[Probability Space]
Let $(\mathcal{A}, \Sigma, P)$ be a discrete probability space over the set of actions $\mathcal{A}$. Let $\pi_{\text{ref}}$ and $\pi^*$ be probability measures on this space.
\end{definition}

\begin{lemma}[KL-Probability Bound]
\label{lem:kl_bound}
Let $\pi^*$ be a Dirac measure concentrated at $a^*$, i.e., $\pi^*(a^*) = 1$. If $D_{\text{KL}}(\pi^* \| \pi) \le \epsilon$, then:
\begin{equation}
\pi(a^*) \ge e^{-\epsilon} \ge 1 - \epsilon
\end{equation}
\end{lemma}

\begin{proof}[Proof of Lemma~\ref{lem:kl_bound}]
Since $\pi^*$ is a Dirac measure:
\begin{equation}
D_{\text{KL}}(\pi^* \| \pi) = \sum_{a} \pi^*(a) \log \frac{\pi^*(a)}{\pi(a)} = 1 \cdot \log \frac{1}{\pi(a^*)} = -\log \pi(a^*)
\end{equation}
Given $D_{\text{KL}}(\pi^* \| \pi) \le \epsilon$, we have $-\log \pi(a^*) \le \epsilon \implies \pi(a^*) \ge e^{-\epsilon}$. Using the inequality $e^{-x} \ge 1-x$ for $x \ge 0$, we get $\pi(a^*) \ge 1 - \epsilon$.
\end{proof}

\begin{proof}[Proof of Proposition~\ref{prop:collapse}]
Fix an optimal trajectory $\tau^*=(h_1^*,a_1^*,\dots,h_T^*,a_T^*)$ and let $h_t^*$ denote the history on this optimal trajectory at step $t$. Under the proposition assumption,
\begin{equation}
D_{\mathrm{KL}}\!\left(\pi^*(\cdot\mid h_t^*)\,\middle\|\,\pi_{\mathrm{ref}}(\cdot\mid h_t^*)\right)\le \epsilon_{\text{sft}}
\quad\Rightarrow\quad
\pi_{\mathrm{ref}}(a_t^*\mid h_t^*) \ge e^{-\epsilon_{\text{sft}}}
\end{equation}
by Lemma~\ref{lem:kl_bound}. Therefore the probability that a single rollout from $\pi_{\mathrm{ref}}$ exactly follows the optimal trajectory is
\begin{equation}
P(\tau=\tau^*) \;=\; \prod_{t=1}^{T}\pi_{\mathrm{ref}}(a_t^*\mid h_t^*) \;\ge\; e^{-T\epsilon_{\text{sft}}} \;\ge\; 1-T\epsilon_{\text{sft}},
\end{equation}
where we used $\prod_{t=1}^T e^{-\epsilon_{\text{sft}}}=e^{-T\epsilon_{\text{sft}}}$ and the inequality $e^{-x}\ge 1-x$ for $x\ge 0$.

Now sample $G$ independent trajectories $\tau_1,\dots,\tau_G \overset{\text{i.i.d.}}{\sim}\pi_{\mathrm{ref}}$. Then
\begin{equation}
P(\tau_1=\cdots=\tau_G=\tau^*) \;=\; P(\tau=\tau^*)^G \;\ge\; e^{-GT\epsilon_{\text{sft}}} \;\ge\; 1-GT\epsilon_{\text{sft}}.
\end{equation}

On the event $\{\tau_1=\cdots=\tau_G\}$, we have $R(\tau_1)=\cdots=R(\tau_G)$, hence $\sigma_g=0$ and $R(\tau_j)-\mu_g=0$ for every $j$. Therefore $A_j=\frac{R(\tau_j)-\mu_g}{\sigma_g+\epsilon_{\text{stab}}}=0$ for all $j$, and the advantage-weighted GRPO policy-gradient term for that group,
\begin{equation}
\frac{1}{G}\sum_{j=1}^G A_j \sum_{t=1}^{T}\nabla_\theta \log \pi_\theta(a_{t,j}\mid h_{t,j})
\end{equation}
is exactly zero.
\end{proof}

\subsection{Proof of Proposition~\ref{prop:variance}}

We first state the key lemmas used in this proof.

\begin{lemma}[Law of Total Variance]
\label{lem:total_variance}
For random variables $X$ and $Y$, the total variance of $X$ decomposes as:
\begin{equation}
\text{Var}(X) = \underbrace{\mathbb{E}_{Y}[\text{Var}(X|Y)]}_{\text{within-group variance}} + \underbrace{\text{Var}_{Y}(\mathbb{E}[X|Y])}_{\text{between-group variance}}
\end{equation}
Since the first term is non-negative, we have $\text{Var}(X) \ge \text{Var}_{Y}(\mathbb{E}[X|Y])$.
\end{lemma}

\begin{lemma}[Expected Group Second Central Moment]
\label{lem:expected_group_var}
Let $X_1,\dots,X_G$ be i.i.d.\ with finite variance $\text{Var}(X)$. Define $\bar{X}=\frac{1}{G}\sum_{j=1}^G X_j$ and $S_G^2=\frac{1}{G}\sum_{j=1}^G (X_j-\bar{X})^2$. Then
\begin{equation}
\mathbb{E}[S_G^2] = \frac{G-1}{G}\,\text{Var}(X).
\end{equation}
\end{lemma}

\begin{proof}
Using $S_G^2=\frac{1}{G}\sum_{j=1}^G X_j^2-\bar{X}^2$ and letting $\mu=\mathbb{E}[X]$, we have
\begin{equation}
\mathbb{E}[S_G^2] = \mathbb{E}[X^2] - \mathbb{E}[\bar{X}^2]
= \mathbb{E}[X^2] - \big(\text{Var}(\bar{X}) + (\mathbb{E}[\bar{X}])^2\big)
= (\text{Var}(X)+\mu^2) - \left(\frac{1}{G}\text{Var}(X)+\mu^2\right),
\end{equation}
which yields $\mathbb{E}[S_G^2]=\frac{G-1}{G}\text{Var}(X)$.
\end{proof}

\begin{proof}[Proof of Proposition~\ref{prop:variance}]
Let $r\in\{\texttt{high},\texttt{low}\}$ be sampled from $P_{\mathrm{sample}}$ with $p:=P_{\mathrm{sample}}(r=\texttt{high})\in(0,1)$. Let $\tau\sim \pi_{\mathrm{ref}}(\cdot\mid r)$ and define the bounded reward random variable $X:=R(\tau)\in[0,1]$. Let $\mu_h:=\mathbb{E}[X\mid r=\texttt{high}]$ and $\mu_l:=\mathbb{E}[X\mid r=\texttt{low}]$.

By Lemma~\ref{lem:total_variance} (law of total variance),
\begin{equation}
\text{Var}(X) \;\ge\; \text{Var}_{r}\!\big(\mathbb{E}[X\mid r]\big)
 \;=\; p(1-p)\,(\mu_h-\mu_l)^2.
\end{equation}
Under the proposition assumption $|\mu_h-\mu_l|\ge \epsilon$, we obtain
\begin{equation}
\text{Var}(X) \ge p(1-p)\,\epsilon^2.
\end{equation}

Now draw a GRPO group by sampling $G$ i.i.d.\ pairs $(r_j,\tau_j)$ with $r_j\sim P_{\mathrm{sample}}$ and $\tau_j\sim \pi_{\mathrm{ref}}(\cdot\mid r_j)$, and define $X_j:=R(\tau_j)$ and $\sigma_g^2:=\frac{1}{G}\sum_{j=1}^G (X_j-\bar{X})^2$ where $\bar{X}=\frac{1}{G}\sum_{j=1}^G X_j$. By Lemma~\ref{lem:expected_group_var},
\begin{equation}
\mathbb{E}[\sigma_g^2] = \frac{G-1}{G}\,\text{Var}(X)
\;\ge\; \frac{G-1}{G}\,p(1-p)\,\epsilon^2.
\end{equation}
Thus we can take $\kappa=\frac{G-1}{G}p(1-p)$.
\end{proof}

\section{The Details of the Dataset Formulations}
\label{app:dataset}

\subsection{Berkeley Function Calling Leaderboard (BFCLv4)}
The Berkeley Function Calling Leaderboard (BFCLv4) \cite{bfcl2024} is a comprehensive benchmark designed to evaluate the tool-calling capabilities of Large Language Models. It encompasses a wide variety of APIs (e.g., Java, JavaScript, Python) and diverse scenarios. Our work primarily focuses on the \textbf{multi-turn dataset} section, which is specifically designed to test an agent's ability to maintain context, handle state changes, and execute sequential tools to achieve a complex goal.

\textbf{Benchmark Statistics.}
The multi-turn subset consists of 200 high-quality, human-curated evaluation trajectories. Each trajectory involves multiple steps (typically 3-10 turns) where the agent must interact with a simulated environment. The primary domains include:
\begin{itemize}
    \item \textbf{GorillaFileSystem:} A simulated Linux file system supporting commands like \texttt{ls}, \texttt{cd}, \texttt{grep}, \texttt{find}, \texttt{mv}, etc.
    \item \textbf{TwitterAPI:} A mock social media API for posting tweets, replying, and managing followers.
    \item \textbf{MathAPI:} Utilities for statistical calculations (mean, standard deviation, logarithm).
    \item \textbf{TicketAPI:} A system for managing support tickets (create, resolve, query).
\end{itemize}

\subsubsection{Data Curation Pipeline}
\label{app:bfcl_curation}

We employ a systematic pipeline to construct aligned SFT, RCTP-FT, and RL datasets for BFCLv4 multi-turn evaluation. All datasets share the same ID-based train/test split (90\%/10\%) to ensure zero question overlap between training and evaluation.

\paragraph{Stage 1: SFT and RCTP-FT Data.}

\textit{Source.} We derive training data from two complementary sources. Since the RCTP-FT training dataset uses an exact 1:1 expert-to-failure ratio (800:800), we set $p=0.5$ in Eq.~\ref{eq:sampling}.
\begin{itemize}
    \item \textbf{Expert Trajectories (800):} Extracted from the official BFCLv4 ground truth (\texttt{possible\_answer/} files). These represent optimal execution paths verified by the benchmark maintainers across four multi-turn categories: \texttt{base}, \texttt{long\_context}, \texttt{miss\_func}, and \texttt{miss\_param} (200 samples each).
    \item \textbf{Failure Trajectories (800):} Collected via RL exploration rollouts using Qwen2.5-7B-Instruct. These capture diverse failure modes including incorrect tool selection, parameter errors, and premature termination.
\end{itemize}

\textit{Split.} We apply an ID-based split: 720 train / 80 test. This ensures that questions sharing the same numeric ID (which appear across multiple categories) are kept together in either train or test, preventing data leakage.

\textit{Format.} Both SFT and RCTP-FT datasets use OpenAI chat format with a combined task prompt (``Step 1: ... Step 2: ... Complete all the steps above using the available tools. Call done() when you have finished all tasks.''). The key distinction:
\begin{itemize}
    \item \textbf{SFT:} Expert trajectories only, no reward tokens.
    \item \textbf{RCTP-FT:} Both expert and failure trajectories, with reward token appended to the first user message: \texttt{[Reward Goal: <|high\_reward|>]} for success and \texttt{[Reward Goal: <|low\_reward|>]} for failure.
\end{itemize}

\textit{Statistics.} Table~\ref{tab:bfcl_stats} summarizes the curated data. The RCTP-FT dataset maintains an exact 50-50 balance between success and failure trajectories, which is essential for learning reward-conditioned generation.

\begin{table}[h]
\centering
\caption{SFT and RCTP-FT Dataset Statistics for BFCLv4}
\label{tab:bfcl_stats}
\small
\begin{tabular}{@{}lcccc@{}}
\toprule
\textbf{Dataset} & \textbf{Train} & \textbf{Test} & \textbf{Total} & \textbf{Success Rate} \\
\midrule
SFT & 720 & 80 & 800 & 100\% \\
\midrule
RCTP-FT (success) & 720 & 80 & 800 & 100\% \\
RCTP-FT (failure) & 720 & 80 & 800 & 0\% \\
\textbf{RCTP-FT Total} & \textbf{1,440} & \textbf{160} & \textbf{1,600} & \textbf{50\%} \\
\bottomrule
\end{tabular}
\end{table}

\paragraph{Stage 2: RL Data.}

\textit{Source.} For online RL, we use minimal task pointers that reference the original BFCLv4 questions. The BFCL environment adapter constructs the full system prompt and combined task prompt at runtime, enabling dynamic interaction with the simulated environment.

\textit{Split.} We use the same ID-based split as Stage 1: 720 train / 80 test. Table~\ref{tab:rl_split} shows the distribution across categories. The per-category counts are not exactly 180/20 because the split operates at the ID level first, and samples sharing an ID inherit that assignment.

\begin{table}[h]
\centering
\caption{RL Data Train/Test Split by Category}
\label{tab:rl_split}
\small
\begin{tabular}{@{}lccc@{}}
\toprule
\textbf{Category} & \textbf{Train} & \textbf{Test} & \textbf{Total} \\
\midrule
multi\_turn\_base & 182 & 18 & 200 \\
multi\_turn\_long\_context & 177 & 23 & 200 \\
multi\_turn\_miss\_func & 183 & 17 & 200 \\
multi\_turn\_miss\_param & 178 & 22 & 200 \\
\midrule
\textbf{Total} & \textbf{720} & \textbf{80} & \textbf{800} \\
\bottomrule
\end{tabular}
\end{table}

\textit{Format.} Each RL sample is a minimal task pointer:
\begin{verbatim}
{
  "conversations": [{"from": "human", "value": "[BFCL task: multi_turn_base_0]"}],
  "task_id": "multi_turn_base_0",
  "source_file": "BFCL_v4_multi_turn_base.json",
  "env_type": "bfcl"
}
\end{verbatim}

\paragraph{Format Alignment.}
All data phases enforce strict format alignment to ensure consistency across SFT, RCTP-FT, and RL:
\begin{enumerate}
    \item \textbf{Message Structure:} OpenAI chat format (\texttt{\{role, content, tool\_calls\}}).
    \item \textbf{Combined Task Prompt:} User instructions are merged into a single ``Step 1: ..., Step 2: ...'' format.
    \item \textbf{Schema Consistency:} All assistant messages include an explicit \texttt{tool\_calls} key (empty list if no calls).
\end{enumerate}

\subsubsection{Data Examples}
\label{app:examples}

We provide examples to illustrate the data formats used in each training stage.

\paragraph{Multi-Turn Trajectory Example.}
This example from \texttt{multi\_turn\_base\_1} demonstrates a task requiring file system manipulation. During SFT/RCTP-FT, the model observes multi-turn user queries as separate turns. During RL, the queries are merged into a single combined task prompt.

\begin{defaultbox}
\textbf{Task ID:} \texttt{multi\_turn\_base\_1} \\
\textbf{User Query 1:} ``I am alex. Check if the current directory is under my name and list all the visible and hidden contents in the current directory now, please.'' \\
\textbf{Agent Action 1:} \texttt{ls(a=True)} \\
\textbf{Observation 1:} \texttt{\{"status": "success"\}}

\textbf{User Query 2:} ``Go to workspace directory and move one of the 'log.txt' files into a new directory 'archive'.'' \\
\textbf{Agent Action 2:} \texttt{cd(folder="workspace")} \\
\textbf{Agent Action 3:} \texttt{mv(source="log.txt", destination="archive")}

\textbf{User Query 3:} ``Investigate within 'log.txt' for the occurrence of the keyword 'Error'.'' \\
\textbf{Agent Action 4:} \texttt{cd(folder="archive")} \\
\textbf{Agent Action 5:} \texttt{grep(file\_name="log.txt", pattern="Error")}
\end{defaultbox}

\paragraph{Reward-Conditioned Format Example.}
For RCTP-FT, we augment trajectories with reward tokens. Below shows how the same task is formatted for high-reward (expert) and low-reward (failure) trajectories:

\begin{orangebox}
\textbf{RCTP-FT Input (High Reward):} \\
\textbf{User:} Find \texttt{analysis\_report.csv} and move it to 'archive'... \\
\texttt{[Reward Goal: <|high\_reward|>]}

\textbf{Assistant:} \texttt{mv(source="analysis\_report.csv", destination="archive")}

\texttt{...}
\end{orangebox}

\begin{bluebox}
\textbf{RCTP-FT Input (Low Reward):} \\
\textbf{User:} Find \texttt{analysis\_report.csv} and move it to 'archive'... \\
\texttt{[Reward Goal: <|low\_reward|>]}

\textbf{Assistant:} \texttt{cp(source="analysis\_report.csv", destination="archive")} \\
\textit{(Copying instead of moving is a common failure mode)}

\texttt{...}
\end{bluebox}

\section{Model Details}
\label{app:model_details}

This appendix summarizes the exact model versions used in our experiments.

\subsection{Base Models (Open-Weights)}
We report the exact model identifiers in Table~\ref{tab:base_model_versions}.
\begin{table}[h]
\centering
\caption{Base model versions used in this paper.}
\label{tab:base_model_versions}
\small
\begin{tabular}{@{}ll@{}}
\toprule
\textbf{Model} & \textbf{Version / Identifier} \\
\midrule
LLaMA-3.1-8B-Instruct & meta-llama/Llama-3.1-8B-Instruct \\
Qwen2.5-7B-Instruct & Qwen/Qwen2.5-7B-Instruct \\
\bottomrule
\end{tabular}
\end{table}

\subsection{API Models}
We report the exact model identifiers in Table~\ref{tab:api_model_versions}.
\begin{table}[h]
\centering
\caption{API model versions evaluated on BFCLv4 (validation).}
\label{tab:api_model_versions}
\small
\begin{tabular}{@{}ll@{}}
\toprule
\textbf{Model} & \textbf{Version / Identifier} \\
\midrule
Opus-4.5 & claude-opus-4-5-20251101 \\
Sonnet-4.5 & claude-sonnet-4-5-20250929 \\
GLM-4.7 & glm-4.7 \\
Gemini-3-Pro & gemini-3-pro-preview \\
GPT-5.2 & gpt-5.2-2025-12-11 \\
\bottomrule
\end{tabular}
\end{table}

\section{Formal Reward Function Definition}
\label{app:reward_details}

This section provides the rigorous mathematical formulation of the trajectory-level reward function $R(\tau)$ used in RC-GRPO (see Sec.~\ref{sec:reward_function}), followed by a concrete calculation example.

\subsection{Mathematical Formulation}

\subsubsection{Notation}
We use the following notation:
\begin{itemize}
    \item $S_{\text{final}}$: The final state of the environment (API states, databases, etc.) after the agent's trajectory $\tau$.
    \item $S_{\text{gold}}$: The final state of the environment after executing the ground truth trajectory.
    \item $A_{\text{traj}}$: The set of all tool calls executed by the agent in trajectory $\tau$.
    \item $A_{\text{gold}}$: The set of all essential tool calls in the ground truth trajectory.
    \item $\mathbbm{1}[C]$: Indicator function (1 if condition $C$ is true, 0 otherwise).
\end{itemize}

For BFCLv4, the total reward $R(\tau)$ is the product of a state consistency check and an action coverage check:
\begin{equation}
    R(\tau) = R_{\text{state}} \cdot R_{\text{action}}
\end{equation}

\subsubsection{Reward Components}

\paragraph{State Consistency ($R_{\text{state}}$).}
This component verifies that the side effects of the agent's actions match the ground truth. This is critical for tasks where different action sequences can yield the same valid outcome (e.g., creating a file).
\begin{equation}
    R_{\text{state}} = \mathbbm{1}[S_{\text{final}} = S_{\text{gold}}]
\end{equation}
In implementation, this involves comparing the hash or deep equality of the environment's state dictionary.

\paragraph{Action Coverage ($R_{\text{action}}$).}
This component ensures that all required actions were performed. Unlike strict sequential matching, this allows for reordering of independent commutative actions.
\begin{equation}
    R_{\text{action}} = \mathbbm{1}\left[ \forall a^* \in A_{\text{gold}}, \exists a \in A_{\text{traj}} \text{ s.t. } \text{Match}(a, a^*) \right]
\end{equation}
where $\text{Match}(a, a^*)$ is true if and only if:
\begin{enumerate}
    \item $a.\texttt{name} = a^*.\texttt{name}$
    \item For every parameter $k, v$ in $a^*.\texttt{args}$, $a.\texttt{args}[k] = v$. (The agent may provide extra optional parameters, but must match all required/golden parameters).
\end{enumerate}

\subsection{Concrete Calculation Example}

Consider a task: ``Move \texttt{report.csv} to \texttt{/archive} and delete \texttt{temp.log}.''

\textbf{Ground Truth ($A_{\text{gold}}$):}
\begin{enumerate}
    \item \texttt{mv(src="report.csv", dst="/archive")}
    \item \texttt{rm(path="temp.log")}
\end{enumerate}

\textbf{Scenario 1: Perfect Execution (Success)}
\begin{itemize}
    \item Agent Actions: \texttt{rm("temp.log")}, then \texttt{mv("report.csv", "/archive")}.
    \item $R_{\text{state}} = 1$ (Filesystem state matches gold).
    \item $R_{\text{action}} = 1$ (Both \texttt{mv} and \texttt{rm} present with correct args, order ignored).
    \item \textbf{Total Reward:} $1 \cdot 1 = 1$.
\end{itemize}

\textbf{Scenario 2: Right Actions, Wrong State (Failure)}
\begin{itemize}
    \item Agent Actions: \texttt{mv("report.csv", "/archive")}, \texttt{rm("temp.log")}, but then \texttt{touch("temp.log")}.
    \item $R_{\text{action}} = 1$ (All essential actions executed).
    \item $R_{\text{state}} = 0$ (\texttt{temp.log} exists in final state, but shouldn't).
    \item \textbf{Total Reward:} $0 \cdot 1 = 0$.
\end{itemize}

\textbf{Scenario 3: Missing Action (Failure)}
\begin{itemize}
    \item Agent Actions: Only \texttt{mv("report.csv", "/archive")}.
    \item $R_{\text{state}} = 0$ (\texttt{temp.log} still exists).
    \item $R_{\text{action}} = 0$ (Missing \texttt{rm} call).
    \item \textbf{Total Reward:} $0 \cdot 0 = 0$.
\end{itemize}

\section{Experimental Settings Details}
\label{app:experiment_settings}

We conduct our experiments using 8 NVIDIA H200 GPUs.

This section details the hyperparameter configurations for RC-GRPO (Table~\ref{tab:rc_grpo_settings}), RCTP-FT (Table~\ref{tab:rctp_ft_settings}), and the main baselines (Table~\ref{tab:baseline_settings}).

\subsection{Proposed Method (RC-GRPO)}

We summarize the hyperparameter configurations for RC-GRPO (Table~\ref{tab:rc_grpo_settings}) and RCTP-FT (Table~\ref{tab:rctp_ft_settings}).

Table~\ref{tab:rc_grpo_settings} lists the hyperparameters used for the RC-GRPO experiments on BFCLv4.

\begin{table}[h]
    \centering
    \caption{Hyperparameter settings for RC-GRPO.}
    \label{tab:rc_grpo_settings}
    \begin{tabular}{lcc}
        \toprule
        \textbf{Hyperparameter} & \textbf{LLaMA-3.1-8B} & \textbf{Qwen2.5-7B} \\
        \midrule
        Learning Rate & $1 \times 10^{-6}$ & $1 \times 10^{-6}$ \\
        KL Coefficient ($\beta$) & 0.1 & 0.1 \\
        Group Size ($G$) & 5 & 5 \\
        Batch Size & 256 & 256 \\
        Clip Ratio & 0.2 & 0.2 \\
        Epochs & 400 & 400 \\
        \bottomrule
    \end{tabular}
\end{table}

\begin{table}[h]
    \centering
    \caption{Hyperparameter settings for RCTP-FT (reward-conditioned trajectory finetuning).}
    \label{tab:rctp_ft_settings}
    \begin{tabular}{lcc}
        \toprule
        \textbf{Hyperparameter} & \textbf{LLaMA-3.1-8B} & \textbf{Qwen2.5-7B} \\
        \midrule
        Learning Rate & $1 \times 10^{-5}$ & $1 \times 10^{-5}$ \\
        Batch Size & 16 & 16 \\
        Epochs & 5.0 & 5.0 \\
        Reward Tokens & \texttt{<|high\_reward|>}, \texttt{<|low\_reward|>} & \texttt{<|high\_reward|>}, \texttt{<|low\_reward|>} \\
        \bottomrule
    \end{tabular}
\end{table}

\subsection{Baselines}

Table~\ref{tab:baseline_settings} summarizes the hyperparameter configurations for the baseline methods.

\begin{table}[h]
    \centering
    \caption{Hyperparameter settings for baseline methods.}
    \label{tab:baseline_settings}
    \begin{tabular}{lcc}
        \toprule
        \textbf{Hyperparameter} & \textbf{LLaMA-3.1-8B} & \textbf{Qwen2.5-7B} \\
        \midrule
        \multicolumn{3}{l}{\textbf{SFT}} \\
        Learning rate & $1 \times 10^{-5}$ & $1 \times 10^{-5}$ \\
        Batch size & 16 & 16 \\
        Epochs & 2.0 & 2.0 \\
        Optimizer & AdamW & AdamW \\
        \midrule
        \multicolumn{3}{l}{\textbf{GRPO}} \\
        Learning rate & $1 \times 10^{-6}$ & $1 \times 10^{-6}$ \\
        Batch size & 64 & 64 \\
        Group size ($G$) & 5 & 5 \\
        KL coefficient ($\beta$) & 0.1 & 0.1 \\
        Clip ratio ($\epsilon$) & 0.2 & 0.2 \\
        \bottomrule
    \end{tabular}
\end{table}


\end{document}